\documentclass[10pt]{article}

\usepackage[margin=1.4in]{geometry}

\usepackage{graphicx}
\usepackage{color}
\usepackage{subfigure}
\usepackage{amsmath}
\usepackage{amsfonts}
\usepackage{microtype}
\usepackage{graphicx}
\usepackage{subfigure}
\usepackage{booktabs} 
\usepackage{hyperref}

\usepackage{amsmath}
\usepackage{amssymb}
\usepackage{footnote}
\usepackage[normalem]{ulem}
\makesavenoteenv{tabular}
\makesavenoteenv{table}
\def\ttheta{\boldsymbol \theta}

\def\ttheta{ {\boldsymbol{\theta}} }
\def\LL{{\mathcal{L}}}
\def\real{ {\rm I\!R} }
\def\natural{ {\rm I\!N} }
\def\OO{ {\mathcal{O}} }
\def\C{ {\mathcal{C}} }
\DeclareMathOperator*{\argmin}{argmin~}
\newtheorem{theorem}{Theorem}

\newtheorem{proof}[theorem]{Proof}

\usepackage[round]{natbib}
\renewcommand{\bibname}{References}
\renewcommand{\bibsection}{\subsubsection*{\bibname}}

\usepackage{authblk}

\title{Deep Demosaicing for Edge Implementation}
\author[1]{Ramchalam Ramakrishnan}
\author[2]{Shangling  Jui}
\author[1]{Vahid Patrovi Nia\thanks{Corresponding author: vahid.partovinia@huawei.com}}
\affil[1]{Huawei Noah's Ark Lab}
\affil[2]{Huawei Kirin Solution  }
\date{}

\begin{document}

\maketitle              
\begin{abstract}



Most digital cameras use sensors coated with a Color Filter Array (CFA) to capture channel components at every pixel location, resulting in a mosaic image that does not contain pixel values in all channels. Current research on reconstructing these missing channels, also known as \emph{demosaicing}, introduces many artifacts, such as \emph{zipper effect} and \emph{false color}. Many deep learning demosaicing techniques  outperform other classical techniques in reducing the impact of artifacts. However, most of these models tend to be over-parametrized. Consequently, edge implementation  of the state-of-the-art deep learning-based demosaicing algorithms on low-end edge devices is a major challenge. We provide an exhaustive search of deep neural network architectures and obtain a Pareto front of Color Peak Signal to Noise Ratio (CPSNR) as the performance criterion versus the number of parameters as the model complexity that outperforms the state-of-the-art. Architectures on the Pareto front can then be used to choose the best architecture for a variety of resource constraints. Simple architecture search methods such as exhaustive search and grid search requires some conditions of the loss function to converge to the optimum. We clarify these conditions in a brief theoretical study.\\

\noindent\textbf{Copyright Notice}: This paper has been accepted in 16th International Conference of Image Analysis and Recognition (ICIAR 2019). Authors have assigned The Springer Verlage all rights under the Springer copyright. The revised version of the article is published in the Lecture Notes in Computer Science Springer series under the same title.

\end{abstract}
\section{Introduction}

%


Deep Learning is changing lives on a day to day basis. Overparametrization of deep models can have an impact in their deployment for variety of applications. For instance, model size could adversely affect the real-time updates in mobile phone implementation. Moreover, the training and inference time increases for a large model. For instance, ResNet152 takes 1.5 weeks to train on the Image-Net dataset using one Nvidia TitanX GPU \citep{Han_EfficientNN_2015}. In terms of hardware, much of the impact of such models is in the energy consumption. As an example, Dynamic Random Access Memory (DRAM) consumes most of the energy at inference. Therefore,  it is imperative to reduce the model size \citep{Dally_EIE_2016}. 


The success of deep learning models has been proven in many high-level vision tasks such as image classification and object detection. However, deep learning solutions are less explored for  low-level vision problems where the edge implementation is an obligation. Previous work shows the effectiveness on other low-level vision problems such as image denoising and image demosaicing \citep{Xie_Denoising_2012}. Here we explore various architectures for demosaicing, with the aim of finding suitable models for deployment on the edge. It is imperative to make use of models that can be stored on edge devices. 


\section{Demosaicing} \label{demosaicing}
Image demosaicing involves interpolating full-resolution images from incomplete color samples produced by an image sensor. Limitations in camera sensor resolution and sensitivity leads to the problem of mosaicing. 
For a digital camera to capture or produce a full color image, it uses sensors to detect all 3 colors (RGB) at every pixel location. One way this is done involves using a beam-splitter to project the image onto three separate sensors, see Figure \ref{fig:CFA}a, one for each color (RGB). Color filters are placed in front of each sensor to filter specific wavelengths 
such that three full-channel color images are obtained. This is a costly process that is generally used in scientific-grade microscopes. Most digital cameras, on the other hand, make use of just one sensor coated with a Color Filter Array (CFA). In this system, one color component is captured at every pixel location and the missing channels are reconstructed. The resultant image is often called a mosaic image, derived from the CFA's mosaic pattern. This mosaic image is then subjected to software-based interpolation, resulting in a full-resolution image; a process termed as \emph{demosaicing}.

\begin{figure}[!h]
  \centering
  \includegraphics[width=0.6\linewidth]{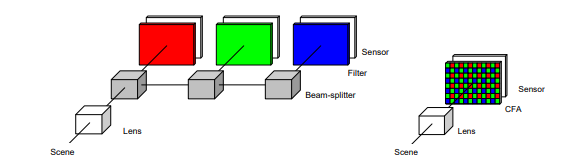}
  \includegraphics[width=0.6\linewidth]{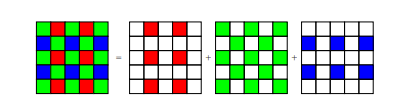}
  \caption{a. Optical path for multi-chip and single-chip digital camera, b. Bayer Pattern in single-chip digital cameras, borrowed from \cite{Li_ImageD_2007}.}
  \label{fig:CFA}
\end{figure}

Various existing CFA patterns are currently used, with the Bayer CFA being the most commonly used approach for mosaicing. The Bayer CFA measures the green image on \emph{quincinx} grid, and the red/blue on the rectangular grids (Figure \ref{fig:CFA}b). The green channel has higher sampling because the human retina has a higher sensitivity toward the green wavelength. 

Let $I^{CFA}:\mathbb{Z}^2 \rightarrow \mathbb{Z}^3$ denote an $M\times N$ Bayer CFA image. If we consider the sampling pattern shown in Figure \ref{fig:CFA}b, then the mosaic image is as follows:
\begin{equation}
   I(i,j) = 
  \begin{cases}
     R_{i,j} & \text{for i odd and j even,}\\
     B_{i,j} & \text{for i even and j odd,}\\
     G_{i,j} & \text{Otherwise,}
  \end{cases}
\end{equation}
\\
where $R_{i,j}$, $G_{i,j}$, $B_{i,j}$ includes values between 0 and 255. To estimate the missing two missing channels for each pixel location using demosaicing, 

\begin{equation}
   \hat{I}(i,j) = 
  \begin{cases}
     (R_{i,j},\hat{B}_{i,j}, \hat{G}_{i,j}) & \text{for i odd and j even,}\\
     (\hat{R}_{i,j},B_{i,j}, \hat{G}_{i,j}) & \text{for i even and j odd,}\\
     (\hat{R}_{i,j},\hat{B}_{i,j}, G_{i,j}) & \text{Otherwise,}
  \end{cases}
\end{equation}
where $\hat{R}_{i,j}$,$\hat{G}_{i,j}$,$\hat{B}_{i,j}$ are the estimates for the channels at each pixel location.

Demosaicing approaches are divided into various categories:
1. Edge-sensitive methods, 2. Directional interpolation and decision methods, 3. Frequency domain approaches, 4. Wavelet-based methods, 5. Statistical reconstruction techniques, and 6. Deep learning \citep{Menon_2011}. Early studies in bi-linear/bi-cubic \citep{Yu_bicubic_2006} and spline interpolation make use of single-channel interpolation techniques that treat each channel separately, without using the inter-channel information and correlation. For instance, bi-linear interpolation computes the red value of a non-red pixel as the average of the 2/4 adjacent red pixels; this method is then repeated with the blue and green pixels. Many comparatively advanced algorithms use the results of bi-linear interpolation as a baseline. These algorithms make use of a weighted average of the neighboring pixels to fill the missing channel values.

Recently, neural network approaches like convolutional neural networks (CNN) have been effectively used for joint denoising and demosaicing \citep{Gharbi_demosaic_2016}. This method down-samples the mosaiced image using multiple convolutions and computes the residual at lower resolution; this is further concatenated with the input mosaic image. The final output is then calculated using another group of convolutions at full resolution to reconstruct the full resolution image. Studies \citep{Tan_2018} have also used Deep Residual Learning for image demosaicing for Bayer CFA. This neural network uses a two-stage architecture: 1) the first step recovers an intermediate result of the Green and Red/Blue channel separately as a guiding prior to use as input for the second stage, 2) the final demosaiced image is reconstructed. 


The current state-of-the-art model, DMCNN-VD \citep{Syu_Demosaicing_2018} architecture (Figure \ref{fig:DMCNN}) consists of a deep network that uses same dimensions for  the input and output layers. They use the Bayer CFA mosaic layer and 20 blocks (CNN, BN, SELU), instead of the 3 layers in the DMCNN model. 


\begin{figure}[h]
  \centering
  \includegraphics[width=0.4\linewidth]{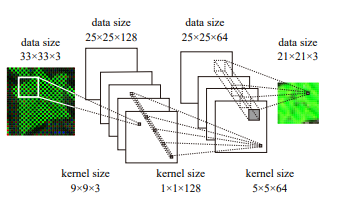}
  \includegraphics[width=0.4\linewidth]{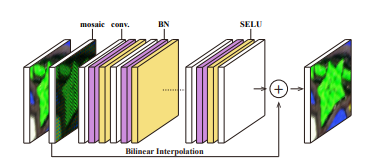}
  \caption{DMCNN  (left) and DMCNN-VD (right) architectures introduced in  \cite{Syu_Demosaicing_2018}.}
  \label{fig:DMCNN}
\end{figure}

\subsection{Artifacts}
Although various approaches to the demosaicing problem have been discussed, most techniques have a propensity to introduce unwanted artifacts into the resulting images. A previous study shows that bi-linear interpolation of mosaic images results in various structural errors \citep{Chang_Bi_2006}, such as the two homogeneous areas with 2 gray levels of  low and high intensity ($L$ and $H$, $L$$<$$H$) represented in Figure \ref{fig:artifact}.1. Another artifact introduced is the mismatch in the green and blue/red channel intensity.

\begin{figure}[!h]
  \centering
  \includegraphics[width=0.35\linewidth]{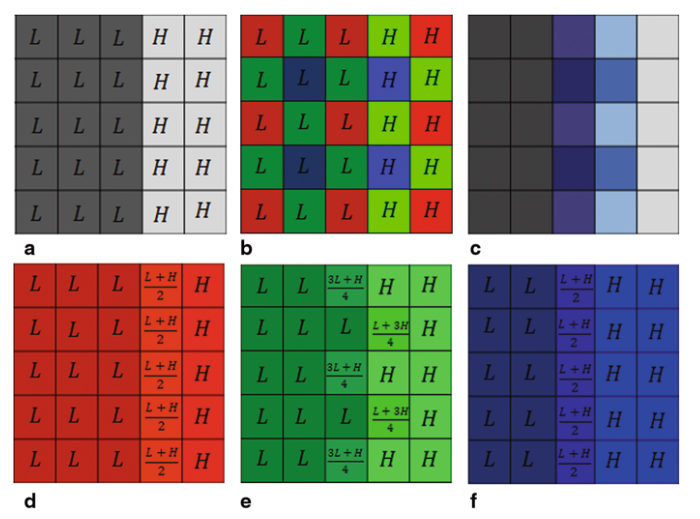}
  \includegraphics[width=0.22\linewidth]{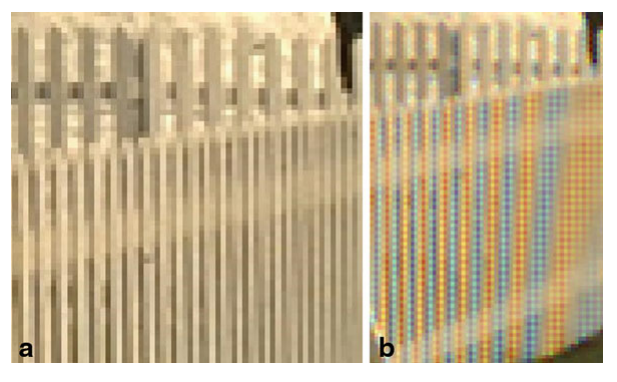}
  \includegraphics[width=0.32\linewidth]{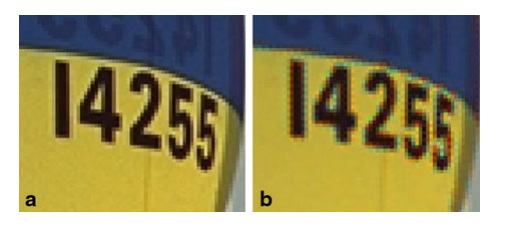}
  \caption{1.) a: Gray image, b: CFA of a, c: Bi-linear interpolation result, d e f: interpolation in color plane, borrowed from \cite{Zhen_Springer_2015}.
  2.)  Zipper Effect and False color - Original image and bi-linear interpolated image.}
  \label{fig:artifact}
\end{figure}

 These 2 artifacts are generalized as \emph{zipper effect} and \emph{false color}. The zipper effect is an abrupt or sudden change in intensity over the neighboring pixels, generally seen along the edges (Figure \ref{fig:artifact}.2a). False colors (Figure \ref{fig:artifact}.2b), on the other hand, are pseudo-colors that are not an accurate representation of the original image. This generally occurs due to inconsistencies in color channels of the original mosaic image. Both these effects are categorized as misguidance color artifacts, generated as a result of wrong interpolation direction. Interestingly, these artifacts are observed even when the interpolation direction is appropriate. However, the effects in this case are less evident when compared to misguided color artifacts. Many deep learning demosaicing techniques outperform other classical techniques in reducing the impact of artifacts. 


\subsection{Performance Evaluation}
A common demosaicing approach is to capture the ground truth images using a professional 3-sensor camera, convert them into a mosaic format using Bayer CFA, interpolate using a demosaicing algorithm, and then compare the results with the ground truth image. Widely used performance criteria are mean Squared Error (MSE) or the equivalent Peak Signal to Noise Ratio (PSNR). The Kodak / McMaster \citep{Yu_McMasterdataset_2018} dataset is generally used as a baseline because of the realistic scenes and varied complexities. The MSE in each color plane is:

\begin{equation}
\mathrm{MSE}(k) = \frac{1}{MN} \sum_{i=1}^M \sum_{j=1}^N \{\hat{I}_k(i,j) - I_k(i,j)\}^2,
\end{equation}
where $\hat{I}_k(i,j)$ is the predicted and $I_k(i,j)$ is the original pixel color component.
Another measure is the Color Mean Squared Error (CMSE) 
\citep{Menon_2011} which averages over the color channel as well,
\begin{equation}
\mathrm{CMSE}(I,\hat{I}) = \frac{1}{3MN} \sum_{k \in \{\mathrm{R}, \mathrm{G}, \mathrm{B}\}} \sum_{i=1}^M \sum_{j=1}^N \{\hat{I}_k(i,j) - I_k(i,j)\}^2,
\end{equation}
The peak signal to noise ratio (PSNR) is defined as,

\begin{equation}
\mathrm{PSNR}(k) = 10  \log_{10}\left(\frac{R^2}{\mathrm{MSE}}\right),
\end{equation}
where R is 1.0 (for float datatype), R is 255 (for integer datatype).
\\
The Color Peak Signal to Noise Ratio (CPSNR) is computed using CMSE. Higher the PSNR value, better is the quality of the demosaiced image.

\section{Architecture Search} \label{ArchitectureSearch}
For the exhaustive search of architectures, we made use of various parameters chosen based on state-of-the-art models to obtain the Pareto front between negative $\mathrm{CPSNR}$ as the loss and the number of parameters as the model complexity. To make the comparisons fair, we ran each model for 500 epochs and retrained the state-of-the-art architecture in this new setting. The various hyper-parameters for exhaustive search were mostly discrete and are as follows:

\begin{enumerate}
\item The use of depthwise separable convolutions \citep{Depthwise_Sifre_2014}, similar to MobileNet V2 architecture \citep{Sandler_MobileNetV2_2018}, over the traditional convolutional layers as per the state-of-the-art models.
\item The number of filters used in the models (16, 32, 64, 128 and 256).
\item Number of blocks, each block consisted of sets of i) depthwise separable ii) standard convolution, and iii) activation function, chosen from (3,5,7 blocks).
\item We use a constant kernel size of 3 $\times$ 3 since much of the hardware has shown to be optimized for this convolution \citep{ShuffleNetV2_Ma_2018}. 
\item The length of skip connections (1,2).
\item The use of cosine annealing \citep{Loshchilov_SGDR_2016}  or a fixed learning rate.
\end{enumerate} 
Some of the reasoning behind choosing the search space is as follows:

\begin{enumerate}
\item Primary reason for choosing this range of filter and block values was to obtain networks that have a smaller number of parameters and FLOPs as the current state-of-the-art networks. The filters and block size dominate the parameter density in a network. We defined an upper bound based on this.
\item Motivated from \cite{ShuffleNetV2_Ma_2018}, choosing  equal channel width minimizes memory access cost (MAC). Hence, we retained the same channel width across all the layers.
\item Furthermore, \citep{ShuffleNetV2_Ma_2018} computed the relative costs of using group convolution, so we used the depth-wise convolution in all the layers.
\item Moreover, for deployment in edge devices, various restrictions are placed based on the platform for inference. In this problem, the model size, the number of FLOPs and MAC play a key role in inference time. For instance, the largest model we have has around $7\times 10^5$ parameters and $5.2\times 10^6$ multiplication FLOPs. Deploying such a model would have severe consequences in the overall usability of the edge device
\item We chose a lower bound on the number of filters, around the region, based on the smallest state-of-the-art model, i.e. DMCNN \citep{Syu_Demosaicing_2018}.
\item We used point-wise convolutions to avoid the use of fully-connected layers as the last layer, and therefore reduce the model size.
\end{enumerate}
Our goal in this study was to explore architectures that improve the state of the art models by using fewer number of parameters and maintaining reasonable accuracy for edge implementation.


\section{Methodology and Theory } \label{SearchMethodology}

Neural architecture search is a discrete optimization problem. Here, we used an exhaustive search over a  small space, where the convergence to the optimum in the designed space is trivial and independent of the loss function, because all points are evaluated. However, most of these neural architecture search techniques are motivated by a very large design space. We show that as the designed space gets infinitely large, even simple methods such as exhaustive search and grid search, require some minimal properties of the loss function to converge to the optimum. Exhaustive search and grid search techniques are common optimization methods that rely only on function evaluation. Therefore, it is counter-intuitive that finding the optimum point using these techniques may require some mathematical assumptions about the search space, or the objective function. Motivated from the derivative-free optimization theory, we aim at clarifying some minimal conditions that the exhaustive search and grid search converges to the optimum.

Let us set the notation first. We denote a vector using bold $\ttheta$ and an element of the this vector using unbold $\theta$. The search space can be a mix of continuous or discrete parameters and denoted by  $\Theta$.
Suppose a neural network is uniquely defined by its hyperparameters embedded in $\ttheta \in \Theta$. Note that the complexity of each network might differ. Complexity may be defined as the number of parameters of the network, inference delay, training time, etc. Often the objective is to discover the neural network with minimum loss (or maximum accuracy) on a fixed validation set. Let's denote the validation loss of the neural network with $\LL(\ttheta)$ and its complexity  $\C(\ttheta)$. It is preferable to minimize $\LL(\ttheta)$ and $\C(\ttheta)$ at the same time, i.e. between two neural networks with the same validation loss  $\LL(.)$, the neural network with smaller complexity $\C(.)$ is preferred. To simplify the theoretical analysis we assume that  $\C(.)$ is fixed and focus on minimizing  the validation loss function $\LL(.) $ regardless of $\C(.)$. In the context of demosaicing $\LL(.)$ is the negative CPSNR and $\C(.)$ is the number of model parameters.

Assume the minimum exists within the search space $\Theta$ 
\begin{equation}
\argmin_{\ttheta\in \Theta} \LL(\ttheta) \neq \emptyset  
\label{eq:minexist}
\end{equation}
Let's define the exhaustive search algorithm to minimize $\LL(.)$  as
\begin{enumerate}
    \item Enumerate: $\{\ttheta_0, \ttheta_1, \ldots, \ttheta_K\}.$
    \item Initialize: $k = 0, \ttheta^*=\ttheta_0, \LL^*=\LL(\ttheta_0).$ 
    \item\label{it:eval} Evaluate  $\LL(\ttheta_{k+1})$.
    \item If $\LL(\ttheta_{k+1})< \LL^*$,  then update $\ttheta^*=\ttheta_{k+1}, \LL^* = \LL(\ttheta_{k+1})$.
    \item If $k<K$ then $k=k+1$ go to \ref{it:eval}, otherwise end.
\end{enumerate}
It is trivial that after iterating the algorithm for $K$ iterations, the global minimum of the objective function is obtained $(\ttheta^*, \LL(\ttheta^*))$, for simplicity of notation we may denote   $(\ttheta^*, \LL^*)$  where
\begin{eqnarray*}    
\ttheta^* = \argmin_{\ttheta \in \Theta} \LL(\ttheta),
\LL^* = \min_{\ttheta \in \Theta} \LL(\ttheta).
\end{eqnarray*}
Convergence to the optimum requires no restriction on the loss function $\LL(.)$ except the existence of the minimum \eqref{eq:minexist}.
However, if the algorithm is iterated only $k<K$ iterations there is no guarantee of convergence to the minimum.

Often the cardinality $||\Theta||=K$ is a large value because in a real scenario  a grid search on the hyperparameter vector $\ttheta$ is implemented, where the search becomes NP-hard, $K=\OO\left(e^d\right)$ where $d=\dim(\ttheta)$, and the algorithm is run only  $k<K$ iterations depending on computation budget.

In the following sections we explore under what conditions the exhaustive search convergence  is assured. Suppose the search space is large but is countable,  $K=\infty$. Infinitely many evaluations requires  more conditions on $\LL(.)$. For simplicity of the notation we may assume $\Theta \subset \real^d$ where $d=\dim(\ttheta)$. Suppose the minimum argument and the objective value  at iteration $k$ of exhaustive search is  denoted by  $(\ttheta^*_k,\LL^*_k)$
\begin{theorem}
\label{theo:conv}
If  $\Theta$ is dense and $\LL(.)$ is continuous,  the exhaustive search converges to the minimum value 
$$\lim_{k\to \infty} \LL^*_k = \LL^*.$$
\end{theorem}
\begin{proof}
Take the optimal solution $\ttheta^* \in \Theta$ and because 
$\LL(.)$ is continuous  and $\Theta$ is dense
$$\exists \delta>0, \ttheta_k\in\Theta_k  \mathrm{~where~} ||\ttheta_k - \ttheta^* || < \delta , |\LL(\ttheta_k)-\LL(\ttheta^*)|<\epsilon, $$
or equivalently $\LL(\theta_k)<\LL(\ttheta^*)+\epsilon$.
On the other hand   $ \forall k, \LL(\ttheta^*) \leq \LL(\ttheta^*_k) \leq \LL(\ttheta_k)$  and by definition $\LL^* = \LL(\ttheta^*), \LL^*_k = \LL(\ttheta^*_k),$ so
$$\LL^* \leq \LL^*_k \leq \LL(\theta_k)\leq \LL^*+\epsilon,$$
$$\implies \lim_{k\to \infty} \LL^*_k = \LL^*.$$
\end{proof}
\begin{theorem}
If $\Theta$ is compact and $\ttheta^*=\argmin_{\ttheta\in \Theta} \LL(\ttheta)$ is unique, the exhaustive algorithm converges to the unique minimum. 
\end{theorem}
\begin{proof}
$\forall k, \ttheta_k\in \Theta$ which is compact, so there is a sub-sequence that converges to $\ttheta^*_k$, call it $\ttheta^*_{ik}$
$$\lim_{k\to \infty}\lim_{i\to\infty} \ttheta^*_{ik}= \tilde \ttheta. $$
By continuity of $\LL(.)$
$$\lim_{k\to\infty} \lim_{i\to\infty} \LL(\ttheta^*_{ik})=  \LL(\tilde{\ttheta}),$$
but from Theorem~\ref{theo:conv} we know $\lim_{k\to\infty} \LL(\ttheta^*_k)$, therefore $\tilde{\ttheta}\in\Theta.$ From the uniqueness of the minimum $\tilde\ttheta=\ttheta^*.$
\end{proof}

The convergence of exhaustive search relies on the continuity of $\LL(.).$ It is easy to construct cases where the exhaustive search fails to converge if $\LL(.)$ is non-continuous \citep{AudetHare_BlackboxBook_2017}. Although exhaustive search is the first tool to explore the discrete space, in a continuous space exhaustive search is in-feasible. Therefore optimizing over a continuous parameter such as learning rate, annealing rate, regularization constant, etc requires other tools. Bayesian optimization \citep{Snoek_BayesianOpt_2012} is often used and some open source packages are available \citep{Tsirigotis_orion_2018}. However, grid search is still one of the most used tool to tune a continuous parameter.

Grid search on multiple variables is cumbersome as a large dimension of the search variable $d$ implies the evaluation set being $\OO(n^d)$. Multivariate grid search algorithm over the set $\prod_{i=1}^d [a_i, b_i] \subset \real^d$ is as follows: 
\begin{enumerate}
    \item Initialize number of grid evaluation $n \in \natural$.
    \item Set the step size $\delta_i = {b_i - a_i \over n-1}, i=1,\ldots, d.$
    \item Create the grid search space with step iteration $j$
    $$\Theta_n = \{\ttheta \mid \ttheta = l_i +  j  \delta_i, j=0,\ldots, n-1, i=1,\ldots, d\}$$
    \item Evaluate the objective function $\LL(.)$ for all $n^d$ grid points $\theta \in \Theta_n$.
    \item Find the grid search solution $(\ttheta^*_n, \LL^*_n)$ where
        $$\ttheta^*_n = \argmin_{\theta \in \Theta_n} \LL(\ttheta), ~~\LL^*_n = \min_{\theta \in \Theta_n} \LL(\ttheta) .$$
\end{enumerate}

Grid search on $d$ variables is NP-hard, so it is commonly  implemented only on one element of the vector $\ttheta$ while keeping the other elements fixed to certain values. Therefore, it makes more sense to explore the convergence of the grid search on a  uni-variate variable $\theta$.

Suppose the continuous uni-variate parameter $\theta\in[a,b]$ achieves its minimum at $(\theta^*,\LL^*).$ The uni-variate parameter $\theta$ is a continuous element, such as  the learning rate, and is also an element of the vector $\ttheta$, in which  the vector $\ttheta$ fully determines the neural network. 

The grid search convergence result relies on creating a nearly dense set $\Theta_n$ by choosing a large number of grid evaluation $n$. The convergence also depends on $\LL(\theta)$ being Lipschitz continuous, i.e.
$$\forall \theta, \theta'\in [a,b],   |\LL(\theta)-\LL(\theta')|\leq M |\theta-\theta'|, $$
where $M$ is the Lipschitz constant.
\begin{theorem}
\label{theo:grid}
Suppose $\LL(.)$ is Lipschitz continuous on $\theta \in [a,b]$ and a grid search is run over $n$ grids to create $\Theta_n$ with step size $\delta$, then
$$ \LL^*_n  - {M\delta\over 2}  \leq \LL^*$$
\end{theorem}
\begin{proof}
Suppose $\theta^*\in\Theta$, take $\ttheta_n \in \Theta_n$ as a grid search candidate. For any point in the grid $\Theta_n$, define the closest point to $\theta^*$ in the grid space
$\theta_n^* = \argmin_{\theta_n\in\Theta_n} |\theta_n-\theta^*|.$
The distance between two consecutive points is $\delta$, so  $|\theta^*_n-\theta^*|\leq {\delta \over 2}.$  By the Lipschitz continuity of $\LL(.)$
$$|\LL(\theta^*_n)-\LL(\theta^*)|\leq M |\theta^*_n-\theta^*| \leq {M\delta \over 2}$$
$$ \implies \LL^*_n  - {M\delta\over 2}  \leq \LL^*.$$
\end{proof}

Theorem~\ref{theo:grid} means for a finite $n$, $\LL^*_n$ may never reach the true minimum $\LL^*$ as $\LL^* \leq \LL^*_n$, but improves swiftly as $\delta$ gets smaller by evaluating a smaller step size or equivalently by increasing the number of evaluations $n$; the lower-bound improves by $\OO(n^{-1})$.

\section{Application} \label{Application}
We first describe the dataset that was used for our experiments and elaborate on the results obtained.
\begin{figure}[t]
  \centering
  \includegraphics[width=0.4\linewidth]{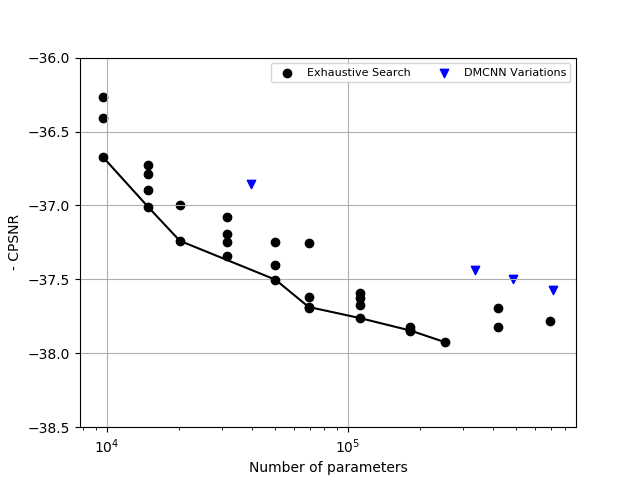}
  \includegraphics[width=0.4\linewidth]{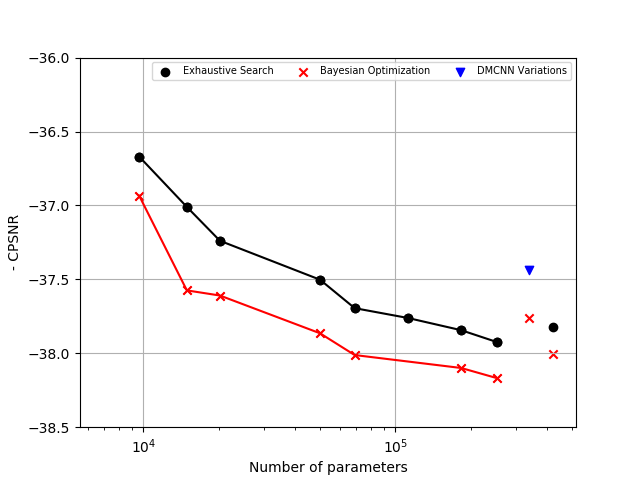}\\
  \includegraphics[width=0.4\linewidth]{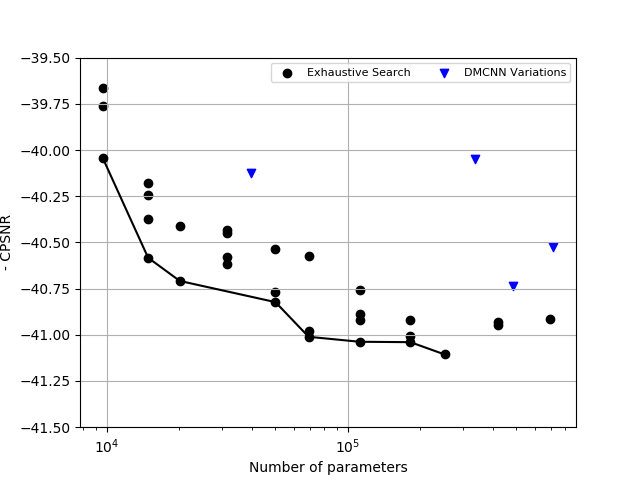}
  \includegraphics[width=0.4\linewidth]{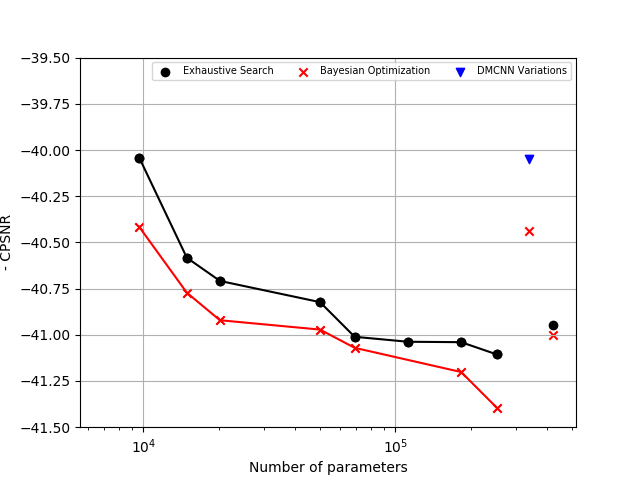}
  \caption{Pareto front of MCM dataset 18 photos with standard error  0.53 (top panels) compared to and Kodak dataset 24 photos standard error  0.66 (bottom panels). Left panels is exhaustive search versus  right  Bayesian Optimization of learning rate and regularization constant implemented on Pareto architectures.}
  \label{fig:PF_McM}
\end{figure}

Flickr500 data is used by \citep{Syu_Demosaicing_2018} to generate around 3.5 million patches. The authors choose the images with some criteria such as i) being colorful enough to obtain the proper color distributions of real world, ii) high-frequency patterns in the images and iii) of high-quality.  However, there are some key points missing in the description to re-generate the exact results such as:
\begin{enumerate}
\item Specifications on exactly how the patches are generated from the original Flickr500 images (related to the stride etc while cropping the patches).
\item Specifications on how to apply the Bayer CFA for creating the mosaic image. 
\end{enumerate}
Therefore, some assumptions were made while generating our data.
\begin{enumerate}
\item We generated random patches of $32\times 32\times 3$ to obtain approximately 1.5 million image patches from the 500 images. Since the method is not specified, we used random points in the image and extracted patches from the random points. 1382400 patches were used for the training set and 153600 patches for the validation set. Note: In the paper, they use patches of size $33 \times  33 \times 3$ but we used $32 \times 32 \times 3$ because the training is much faster on GPU due the dimensions being a power of 2 for tiling.
\item The Bayer CFA generation is also based on the python package. We used the default setting of $RGGB$ and filled the missing channels with zeros to maintain the channel dimensions.
\item All the results in the section were calculated against the McMaster dataset (MCM) containing 18 photos (58000 patches created), and the Kodak images containing 24 photos (78408 patches created).
\end{enumerate}

The Pareto front of our exhaustive search is shown in Figure \ref{fig:PF_McM}. The $y$-axis represents the validation loss $\LL(.)$ which indicates the negative CPSNR and the $x$-axis represents the model complexity $\C(.)$, which is the total number of parameters in the model. The standard error for each experiment is calculated and as they are comparable, they produce a similar prediction interval.

Most of the architectures in the search space outperform the current state-of-the-art (SOTA) as in Figure \ref{fig:PF_McM} in a similar setting, see also Table~\ref{tab:mcm}. The total number of parameters in the current SOTA is approximately $7\times 10^5$. Architectures on the Pareto front maintain the same accuracy while significantly reducing the number of parameters. In our search the effect of cosine annealing was not obvious and would need further tuning on cycle rate. The ResNet skip connections length of 2 blocks worked better than 1.

Some of the models on the Pareto front are shown in the table below in comparison to the current state-of-the-art models for the MCM dataset. Once the initial search was completed, the models from the Pareto front architectures were fed through a Bayesian optimizer using the ORION package \citep{Tsirigotis_orion_2018}, to optimize the learning rate and regularization constant for $L_2$ regularization. Figure~\ref{fig:PF_McM} indicates further improvement over the initial CPSNR obtained with a fixed learning rate $10^{-4}$ and regularization constant $10^{-8}$, so Bayesian optimization is recommended as an extra step after the architecture search as including it in the first step would further increase the complexity of the search. In our experiments, the PSNR was calculated separately for each image in the test dataset; the final CPSNR value was the mean of all the PSNRs of all images in the dataset.

\begin{table}[t]
\begin{footnotesize}
\begin{center}
\begin{tabular}{|  c | c | c |  c | c | c | c | }
\hline
Model& No of blocks & Skip length  & $\LL(.)$ & Standard Error & $\C(.)$ \\ \hline
\hline
\textbf{64 filter} & {\bf 20} & {\bf 20}  & {\bf $-37.58$} & {\bf 0.66} & {\bf 710275}\\
\hline
64 filter& 15 & 15  & $-37.51$ & 0.62 & 487171\\
\hline
64 filter& 10 & 10  & $-37.45$ & 0.64 & 338435\\
\hline
256 filter & 3 & 2 & $-37.77$ & 0.68& 420879\\
\hline
128 filter & 7 & 2  & $-37.94$ & 0.67& 252431\\
\hline
128 filter & 5 & 2  & $-37.82$ & 0.69& 182287\\
\hline
\end{tabular}
\end{center}
\end{footnotesize}
\caption{The negative CPSNR denoted by $\LL(.)$ versus the number of parameters denoted by $\C(.)$ for few  architectures of the search space on the MCM dataset. The first line in bold shows the state-of-the-art.}
\label{tab:mcm}
\end{table}

\section{Conclusion}
Our designed space with a simple exhaustive search outperforms the state of the art and brings a range of loss versus complexity for edge implementation with varying resource constraints. Although in most of architecture search, the number of evaluations and complex search algorithm implementation is the bottleneck, here we showed using prior vision domain knowledge, we can overcome these drawbacks and a simple exhaustive search becomes an effective search tool. We only focused on demosaicing as an example of a low-vision problem where edge implementation is an obligation. We wonder whether a similar approach is useful in edge implementation of other low vision problems. A simple theoretical study suggests the continuity of the loss function and compactness of the search space for exhaustive and grid search to converge to the optimum. 


\section*{Acknowledgement}
We want to acknowledge Jingwei Chen, Yanhui Geng, and Heng Liao;  executing this research project was impossible without their continuous support. We appreciate the assistance of Huawei media research lab of Japan that provided us details about demosaicing and camera sensors. We  declare our deep gratitude to  Charles Audet and S\'ebastien Le Digabel who introduced to us the theory behind derivative-free optimization.

\bibliography{samplepaper}

\end{document}